\documentclass{article}

\usepackage{microtype}
\usepackage{graphicx}
\usepackage{subfigure}
\usepackage{booktabs} %

\usepackage{hyperref}

\usepackage[accepted]{icml2025}

\usepackage{amsmath}
\usepackage{amssymb}
\usepackage{mathtools}
\usepackage{amsthm}
\usepackage{thmtools}
\usepackage{thm-restate}
\usepackage{url}

\DeclareMathOperator*{\arginf}{\arg\!\inf}

\usepackage[capitalize,noabbrev]{cleveref}

\theoremstyle{plain}
\newtheorem{theorem}{Theorem}[section]

\theoremstyle{definition}
\newtheorem{definition}[theorem]{Definition}

\theoremstyle{remark}
\newtheorem{remark}[theorem]{Remark}

\usepackage[textsize=tiny]{todonotes}

\usepackage{bm}
\usepackage{colortbl}
\usepackage{multicol}
\usepackage{multirow}
\usepackage{algorithm}
\usepackage{algorithmic}  

\usepackage{dsfont}

\icmltitlerunning{Learning Relational Tabular Data without Shared Features}

\begin{document}

\twocolumn[
\icmltitle{Learning Relational Tabular Data without Shared Features}

\icmlsetsymbol{equal}{*}

\begin{icmlauthorlist}
	\icmlauthor{Zhaomin Wu}{ids}
	\icmlauthor{Shida Wang}{math}
	\icmlauthor{Ziyang Wang}{soc}
	\icmlauthor{Bingsheng He}{soc}
\end{icmlauthorlist}

\icmlaffiliation{soc}{Department of Computer Science, National University of Singapore, Singapore}
\icmlaffiliation{ids}{Institute of Data Science, National University of Singapore, Singapore}
\icmlaffiliation{math}{Department of Mathematics, National University of Singapore, Singapore}

\icmlcorrespondingauthor{Zhaomin Wu}{zhaomin@nus.edu.sg}

\icmlkeywords{Entity Alignment, Tabular Data, Transformer, Record Linkage, Entity Resolution}

\vskip 0.3in]

\printAffiliationsAndNotice{}  %

\begin{abstract}
    Learning relational tabular data has gained significant attention recently, but most studies focus on single tables, overlooking the potential of cross-table learning. Cross-table learning, especially in scenarios where tables lack shared features and pre-aligned data, offers vast opportunities but also introduces substantial challenges. The alignment space is immense, and determining accurate alignments between tables is highly complex. We propose Latent Entity Alignment Learning (\textit{Leal}), a novel framework enabling effective cross-table training without requiring shared features or pre-aligned data. Leal operates on the principle that properly aligned data yield lower loss than misaligned data, a concept embodied in its soft alignment mechanism. This mechanism is coupled with a differentiable cluster sampler module, ensuring efficient scaling to large relational tables. Furthermore, we provide a theoretical proof of the cluster sampler's approximation capacity. Extensive experiments on five real-world and five synthetic datasets show that Leal achieves up to a 26.8\% improvement in predictive performance compared to state-of-the-art methods, demonstrating its effectiveness and scalability.
\end{abstract}

\section{Introduction}\label{sec:introduction}
Tabular data is a prevalent structured data format, especially in real-world databases. Training models on such data has numerous applications across domains, including medical and financial fields. However, real-world tabular data is often highly heterogeneous, with each table containing a distinct set of features and unique data distributions. This challenge is commonly referred to as the data lake problem~\cite{nargesian2019data} or the data silo problem~\cite{patel2019bridging}. Existing machine learning approaches~\cite{gorishniy2021revisiting, gorishniy2024tabm, chen2016xgboost, prokhorenkova2018catboost} primarily focus on learning from individual tables in isolation. In practice, however, tables are often correlated, and leveraging information across tables can significantly enhance predictive performance. For instance, in the financial domain, data from a transaction table can improve predictions in a loan table by joining the tables on shared features - referred to as \textit{keys} in relational databases - and subsequently training models on the joined table.

The join-learn paradigm is effective in ideal relational databases but faces significant challenges with heterogeneous tabular data widely existed in practice. The primary limitation is the absence of shared features. For example, in an anonymous Bitcoin transaction table and a bank transaction table (Figure~\ref{fig:app-example}), no shared features exist to facilitate a join. Even when tables share features, identifying them by schema or column names is difficult, with exact match accuracy reaching only 18.4\% according to a previous study \cite{vogel2024wikidbs}. While some methods address partially \cite{kang2022fedcvt, sun2023communication} or fuzzily \cite{wu2022coupled, wu2024federated} aligned keys, they can not handle scenarios with no shared features, a scenario we refer to as \textit{latent alignment learning}. 

The second limitation pertains to efficiency, particularly in scenarios involving many-to-many key relationships. Such cases often leads to significant time cost and high memory consumption. In our study, even joining two tables with 200k and 400k rows without optimization takes approximately four hours, with a memory cost 22$\times$ higher than that of the individual tables. These challenges hinder the scalability and practical adoption of join-based learning methods for heterogeneous tabular data in real applications.

\begin{figure}[t!]
    \centering
    \includegraphics[width=0.95\linewidth]{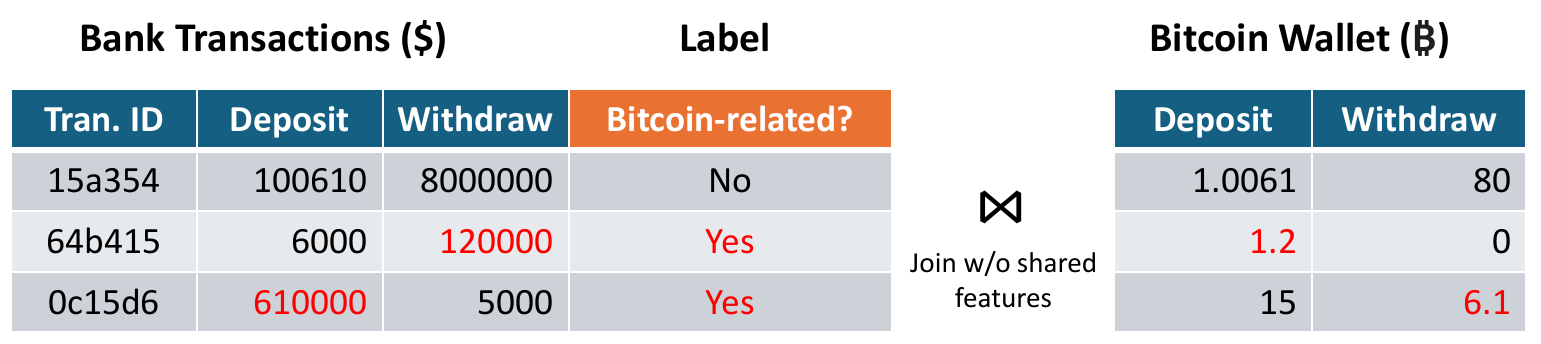}
    \caption{Example of latent alignment learning in a financial application: (left) bank transaction table and (right) Bitcoin transaction table. The tables lack shared features, and the label “Bitcoin-related” can be inferred if bank withdrawals are proportional to Bitcoin deposits (or vice versa).}
    \label{fig:app-example}
\end{figure}

A significant challenge arises to enable the training of relational tabular data without shared features. The absence of keys prevents the evaluation of match probabilities based on key similarities, as traditionally done in relational databases \cite{mishra1992join}. Therefore, identifying new criteria to estimate match probabilities between records across tables becomes crucial. Moreover, designing a model that can effectively learn based on such criteria introduces additional complexity.

To address the above challenges, we propose using training loss decay to estimate match probabilities, based on the insight that \textit{properly aligned data result in smaller loss than misaligned data}. This is formally established in Theorem~\ref{thm:alignment} and validated empirically in Figure~\ref{fig:motivation}. Matched record pairs are probabilistically identified as those yielding the greatest loss decay. Candidate records are embedded and compared via an attention mechanism to assess alignment. Experimental results show that the soft alignment mechanism achieves performance comparable to perfectly aligned data when candidate records contain the ground truth.

\begin{figure}[ht]
    \centering
    \includegraphics[width=0.9\linewidth]{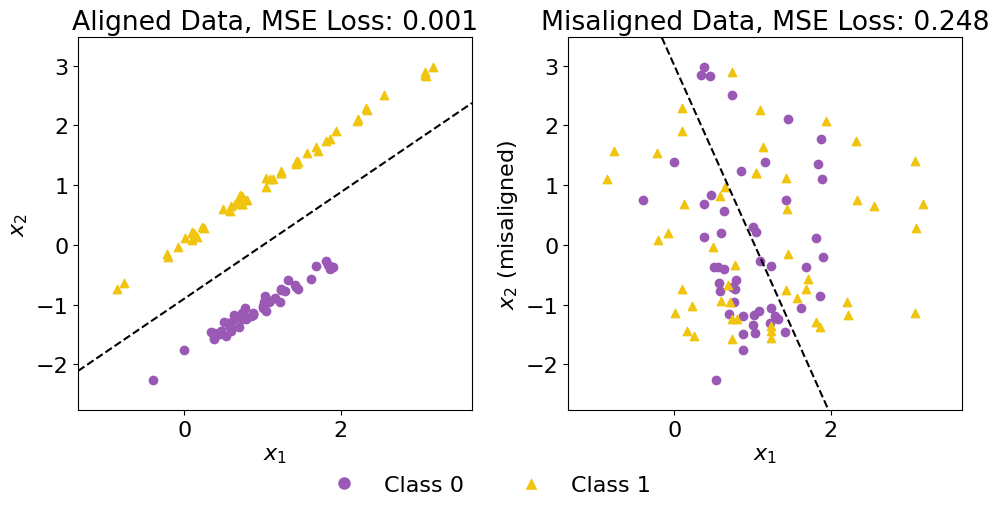}
    \caption{Relationship between data alignment and model training loss. Consider the training of a model on two relational tables: the first table contains a single feature \(x_1\) and a binary label \(y\), while the second table contains a single feature \(x_2\). The left subfigure illustrates the distribution of data points \((x_1, x_2)\) when \(x_1\) and \(x_2\) are properly aligned, while the right subfigure depicts the case of misaligned data. Both subfigures present the converged mean-squared loss and the decision boundary of linear regression.}
    \label{fig:motivation}
\end{figure}

Soft alignment, despite its advantages, remains insufficient for achieving practical latent alignment learning due to scalability challenges. As the number of candidate pairs increases, it becomes prone to overfitting and incurs significant computational costs. To overcome this challenge, we propose a novel module, the \textit{cluster sampler}, which selects a small subset of data records from the table. This module organizes data records into clusters using a soft, differentiable clustering approach and samples records from each cluster. The cluster assignments are dynamically updated during training through gradients propagated from subsequent modules. By limiting the number of candidate pairs, the \textit{cluster sampler} significantly enhances the model's scalability, enabling its effective application to larger tables.

In summary, we propose Latent Entity Alignment Learning (Leal), a coupled model that integrates alignment into supervised learning. Leal addresses the challenge of missing shared features by proposing soft alignment, which learns alignment through training loss. To further improve efficiency, Leal introduces the \textit{cluster sampler}, designed to mitigate overfitting and reduce computational costs when applied to larger tables. To the best of our knowledge, Leal is the first model to enable machine learning across relational tabular data without shared features and pre-aligned samples. The contributions of this paper can be summarized as follows:

\begin{itemize}
    \item We propose a novel approach, Leal, to enable machine learning across relational tabular data without shared features or pre-aligned samples.
    \item We provide theoretical demonstrations to the advantages of aligned data over misaligned data in terms of loss and the approximation capacity of the cluster sampler design.
    \item We evaluate Leal on five real-world and five synthetic datasets, demonstrating its effectiveness in learning relational tabular data without shared features. Our experiments show that Leal reduces the error by up to 26.8\% over state-of-the-art methods trained on individual tables with reasonable computational overhead.
\end{itemize}

\section{Related Work}\label{sec:related_work}
\paragraph{Machine Learning on Tabular Data.} Machine learning on single tables is a well-studied area with diverse paradigms. Deep tabular learning methods~\cite{gorishniy2024tabm, gorishniy2021revisiting, ye2024ptarl} utilize deep neural networks to learn tabular data representations, while tree-based approaches like XGBoost~\cite{chen2016xgboost} and CatBoost~\cite{prokhorenkova2018catboost} rely on gradient boosting decision trees. \citet{mcelfresh2024neural} observed that the performance of these methods varies across datasets. The extension of such models to multiple tables without shared features remains an unexplored area of research.

\paragraph{Multi-Modal Learning.} Multimodal learning~\cite{xu2023multimodal} focuses on integrating information from multiple data modalities, such as images and text. Most multimodal learning algorithms, including VisualBERT~\cite{li2019visualbert} and VL-BERT~\cite{su2019vl}, rely on pre-aligned pairs to supervise alignment, which is incompatible with the Leal setting. Other multimodal learning methods, such as U-VisualBERT~\cite{li2020unsupervised} and VLMixer~\cite{wang2022vlmixer}, achieve self-supervised learning using externally pretrained models for specific modalities to extract features or representations for alignment. However, these approaches are not applicable to heterogeneous tabular data, where universally effective pretrained models are unavailable. None of these multi-modal models can be directly applied to latent alignment learning between relational tabular data.

\paragraph{Vertical Federated Learning.} Vertical Federated Learning (VFL)~\cite{liu2024vertical} is a privacy-preserving learning paradigm for training models across distributed datasets with distinct feature sets. Although privacy is not a concern in the context of Leal, techniques developed to address data heterogeneity in VFL are relevant. Semi-supervised VFL~\cite{kang2022fedcvt, sun2023communication} focuses on the challenge of partial data alignment. They learns to complete missing data based on the aligned data. Fuzzy VFL~\cite{wu2022coupled, wu2024federated} handles scenarios when keys are not precise. \citet{wu2022coupled} use key similarity as the weight for aligned records. \citet{wu2024federated} encode the keys into data embedding by postional encoding~\cite{li2021learnable} and use attention mechanism to align the data. However, these methods rely on the assumption of shared features across datasets, which does not hold in latent alignment learning.

\section{Problem Formulation}\label{sec:problem}
Consider two tables, \(\mathbf{X}^P \in \mathbb{R}^{n^P \times m^P}\) and \(\mathbf{X}^S \in \mathbb{R}^{n^S \times m^S}\), where \(n^P\) and \(n^S\) represent the number of records in the primary table \(\mathbf{X}^P\) and the secondary table \(\mathbf{X}^S\), respectively, and \(m^P\) and \(m^S\) denote the number of features in these tables. We focus on a supervised learning task aimed at predicting a target variable \(\mathbf{y} \in \mathbb{R}^{n^P}\). Without loss of generality, we assume that \(\mathbf{y}\) is associated with \(\mathbf{X}^P\), referred to as the \textit{primary table}, while \(\mathbf{X}^S\) serves as the \textit{secondary table}.

We address a scenario where there are no shared features between \(\mathbf{X}^P\) and \(\mathbf{X}^S\), formally stated as \(a^P \cap a^S = \emptyset\), where \(a^P\) and \(a^S\) represent the feature sets of \(\mathbf{X}^P\) and \(\mathbf{X}^S\), respectively. Despite this, it is assumed that an approximate functional dependency \cite{mandros2017discovering} exists between the primary and secondary tables. This dependency is quantified using the information fraction (IF) metric \cite{reimherr2013quantifying}, defined as:

\[
\text{IF}(Y; X^S \mid X^P) = \frac{I(\mathbf{y}; \mathbf{X}^S \mid \mathbf{X}^P)}{H(\mathbf{y} \mid \mathbf{X}^P)}
\]

where \(I(\mathbf{y}; \mathbf{X}^S \mid \mathbf{X}^P) = H(\mathbf{y} \mid \mathbf{X}^P) - H(\mathbf{y} \mid \mathbf{X}^P, \mathbf{X}^S)\) is the conditional mutual information between \(\mathbf{y}\) and \(\mathbf{X}^S\) given \(\mathbf{X}^P\), and \(H(\mathbf{y} \mid \mathbf{X}^P)\) and \(H(\mathbf{y} \mid \mathbf{X}^P, \mathbf{X}^S)\) denote the conditional entropies of \(\mathbf{y}\) given \(\mathbf{X}^P\) and both \(\mathbf{X}^P\) and \(\mathbf{X}^S\), respectively. The IF measures the additional information provided by \(\mathbf{X}^S\) about \(\mathbf{y}\) beyond what is captured by \(\mathbf{X}^P\). Specifically, \(\text{IF}(Y; X^S \mid X^P) = 1\) implies a functional dependency, while \(\text{IF}(Y; X^S \mid X^P) = 0\) indicates statistical independence. We consider the case where \(|\text{IF}(Y; X^S \mid X^P) - 1|<\delta\), where \(\delta\) is a small positive constant, indicating a strong functional dependency between the primary and secondary tables.

The objective is to learn a predictive model by minimizing the following loss function:

\[
\min_{\theta} \frac{1}{n^P} \sum_{i=1}^{n^P} \mathcal{L}\left(f(\theta; x_i^P, \mathbf{X}^S), y_i\right)
\]

where \(f(\theta; x_i^P, \mathbf{X}^S)\) is a model parameterized by \(\theta\), \(\mathcal{L}\) is a loss function, and \(y_i\) is the target for the \(i\)-th record in \(\mathbf{X}^P\). The challenge lies in effectively utilizing the secondary table \(\mathbf{X}^S\) to enhance the predictive performance for \(\mathbf{y}\), despite the lack of common features between the two tables.

\section{Approach}\label{sec:approach}
In this section, we present the overall framework design of \textbf{L}atent \textbf{E}ntity \textbf{A}lignment \textbf{L}earning (Leal), with the model structure illustrated in Figure~\ref{fig:leal-framework}. Section~\ref{subsec:soft-alignment} introduces the soft alignment mechanism, which maps primary and secondary data records to a shared latent space. In Section~\ref{subsec:cluster-sampler}, we describe the cluster sampler, which dynamically selects the most relevant secondary records during training. Finally, Section~\ref{subsec:train-infer} outlines the training and inference processes.

\begin{figure*}[t!]
    \centering
    \includegraphics[width=0.9\linewidth]{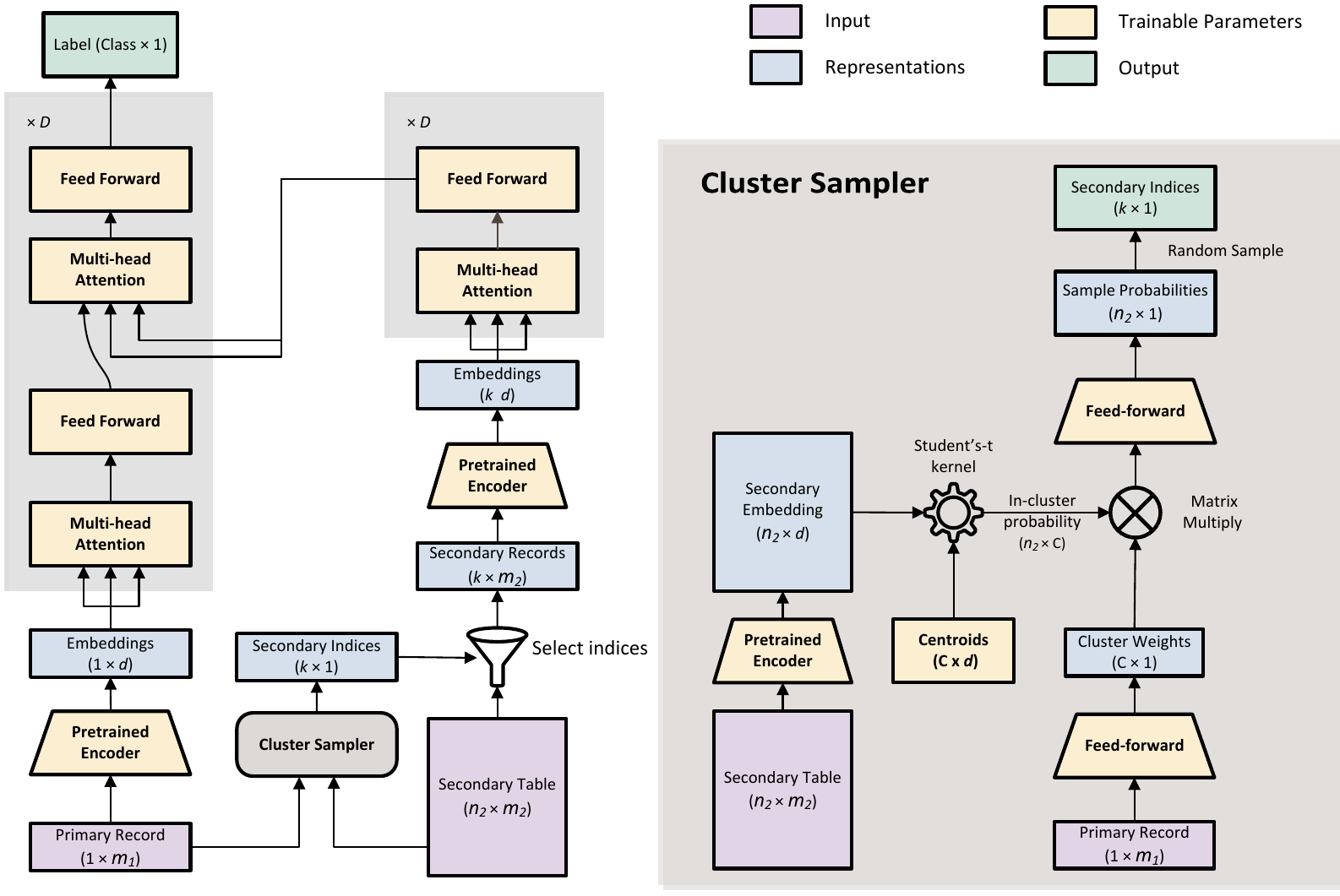}
    \caption{Overall model structure of Leal}
    \label{fig:leal-framework}
\end{figure*}

\subsection{Soft Alignment}\label{subsec:soft-alignment}

The soft alignment mechanism is motivated by the observation that \textit{properly aligned data result in smaller loss compared to misaligned data}. We provide a formal expression for linear regression in Theorem~\ref{thm:alignment}, with the proof detailed in Appendix~\ref{subsec:proof-align-misalign}.
\begin{restatable}{theorem}{thmalignment}\label{thm:alignment}
    Let $\mathbf{X}^P\in \mathbb{R}^{n \times m^P}$ and $\mathbf{X}^S\in \mathbb{R}^{n \times m^S}$ be normalized primary and secondary feature matrices, respectively, $\mathbf{y}\in \mathbb{R}^{n}$ be the target variable, and $\mathbf{R}$ be a permutation matrix. For the linear regression model $\mathbf{y} = \mathbf{X}^P \boldsymbol{\alpha} + \mathbf{R}\mathbf{X}^S \boldsymbol{\beta}$, it holds that $\mathrm{MSE}_{\mathrm{aligned}} \leq \mathrm{MSE}_{\mathrm{misaligned}}$, where $\mathrm{MSE}_{\mathrm{aligned}}$ and $\mathrm{MSE}_{\mathrm{misaligned}}$ are the mean squared errors under correct alignment and random alignment.
\end{restatable}

Theorem~\ref{thm:alignment} suggests that, for linear regression tasks with sufficient samples, properly aligned data consistently results in lower loss. While extending this theoretical result to deep learning remains an open challenge, our experimental results in Figure~\ref{fig:motivation} demonstrate that when \(n\) is sufficiently large that prevents the model from memorizing the data, even complex deep learning models struggle to achieve low loss when the data is misaligned.

To perform soft alignment between a primary data record \(x^P_i\) and \(K\) secondary data records \(\mathbf{x}^S_i = \{x^S_j\}_{j=1}^K\) with distinct features, we first map them into a shared latent space:
\begin{equation}
z^P_i = f^P(x^P_i), \quad \mathbf{z}^S = f^S(\mathbf{x}^S)
\end{equation}
In this latent space, the distance between \(z^P_i \in \mathbb{R}^{1 \times d}\) and each row vector of \(\mathbf{z}^S \in \mathbb{R}^{K \times d}\) represents the relationship between \(x^P_i\) and each record in \(\mathbf{x}^S\), where \(d\) is the dimensionality of the latent space, and \(K\) is the number of candidate records from the secondary table. The scaled inner product between \(z^P_i\) and \(\mathbf{z}^S\) is computed to measure their similarities. The similarity scores for all secondary records are normalized using a softmax function to obtain the soft alignment weights \(\boldsymbol{\lambda}_i\):
\begin{equation}
\boldsymbol{\lambda}_i = \text{SoftMax}\left(\frac{z^P_i (\mathbf{z}^S)^T}{\sqrt{d}}\right)
\end{equation}
The alignment weights \(\boldsymbol{\lambda}_i\) indicate the importance of each record pair and are applied to the latent vectors. Weighted latent vectors are then aggregated for further processing:
\begin{equation}
\tilde{z}^P_i = z^P_i + \sum_{j=1}^K \lambda_{ij} z^S_j
\end{equation}
This weighted alignment module is functionally equivalent to the widely used attention mechanism. Consequently, we directly adopt multi-head attention \cite{vaswani2017attention} in our implementation. Specifically, both primary and secondary representations are processed through a self-attention module followed by a feed-forward network. Their intermediate representations are then aggregated using an attention layer, followed by another feed-forward network. This structure can be stacked to form a deep neural network.

\subsection{Cluster Sampler}\label{subsec:cluster-sampler}

The cluster sampler is designed to efficiently select \(K\) candidate secondary records to feed into the model. For small secondary tables, such as those with hundreds of records, directly feeding the entire table into the model is feasible and demonstrates promising performance in our experiments. However, for large secondary tables, this approach becomes inefficient and may lead to overfitting.

To address this, we propose a trainable cluster sampler to dynamically sample the most relevant records during each training step. The cluster sampler selects \(K\) records based on two criteria: cluster weights and in-cluster probabilities.
The cluster weight is generated by a trainable cluster weight generator, implemented as a feed-forward network that takes the primary record \(x^P_i\) as input and outputs a \(C\)-dimensional vector, where \(C\) is the number of clusters. The cluster weight is normalized using a softmax function to produce the cluster weight vector \(\mathbf{w}_i \in \mathbb{R}^{1 \times C}\):
\begin{equation}
\mathbf{w}_i = \text{SoftMax}(\text{FeedForward}(x^P_i))
\end{equation}
The in-cluster probability is generated using a trainable soft deep K-Means module, inspired by \citet{xie2016unsupervised,ye2024ptarl}. A pretrained encoder maps the secondary table \(\mathbf{x}^S\) into a latent space, resulting in \(\mathbf{h}^S = g^S(\mathbf{x}^S) \in \mathbb{R}^{n_2 \times d}\). Cluster centroids \(\mathbf{C} \in \mathbb{R}^{C \times d}\) are initialized via K-Means clustering on \(\mathbf{h}^S\). Following \citet{xie2016unsupervised}, each record is assigned to a cluster with a probability matrix \(\mathbf{q} \in \mathbb{R}^{n_2 \times C}\), computed using the Student's t-distribution kernel:
\begin{equation}
    q_{ij} = \frac{(1 + ||h^S_j - c_i||^2/\gamma)^{-(\gamma+1)/2}}{\sum_{t=1}^C (1 + ||h^S_j - c_t||^2/\gamma)^{-(\gamma+1)/2}}
\end{equation}
where \(\gamma\) is the degree of freedom, and \(c_i\) and \(h^S_j\) represent the \(i\)-th cluster centroid and the \(j\)-th latent vector, respectively.
The final sampling probability for \(x^P_i\) is computed as the product of cluster weight and in-cluster probability, followed by a multi-layer perceptron (MLP):
\begin{equation}
    \mathbf{p}_i = \mathrm{MLP}(\mathbf{q} \mathbf{w}_i^T) \in \mathbb{R}^{n_2 \times 1} 
\end{equation}
The \(K\) secondary records are sampled based on probabilities \(\mathbf{p}_i\) during training and selected from the top \(\mathbf{p}_i\) during inference. The embeddings of these secondary records are then fed into the soft alignment module.

Furthermore, we theoretically demonstrate that the trainable weights in the cluster sampler design are capable of approximating any target sampling function in Theorem~\ref{thm:cluster-sampler}, which is formally proven in Appendix~\ref{sec:proof}.
\begin{restatable}{theorem}{thmclustersampler}\label{thm:cluster-sampler}
    For any uniform equi-continuous fixed optimal cluster sampler $h^*$ and any $\epsilon > 0$, there exists $d, C$ 
    and a corresponding weight $\theta$ for cluster sampler $h_{CS}$ such that 
    \begin{align}
        \sup_{x \in \mathbf{X}^P} |h^*(x) - h_{CS}(x; \theta)| \leq \epsilon. 
    \end{align}
\end{restatable}

\subsection{Training and Inference}\label{subsec:train-infer}

The training process comprises two stages: an unsupervised stage and a supervised stage. The unsupervised stage focuses on learning an initial representation to facilitate cluster initialization, while the supervised stage is designed to train the model to predict the primary target.

\paragraph{Training.} 
In the unsupervised stage, the pretrained encoder is trained using an autoencoder \cite{zhai2018autoencoder}, where both the encoder \(g^S\) and decoder \(\phi^S\) are simple multi-layer perceptrons (MLPs) with LayerNorm \cite{ba2016layer}:
\begin{equation}
\min_{g^S,\phi^S} \|\mathbf{X}^S - \phi^S(g^S(\mathbf{X}^S))\|_2^2
\end{equation}
In the supervised stage, during each epoch, each batch of data records from the primary table, referred to as the \textit{primary records}, are fed into the cluster sampler alongside the entire secondary table. The sampling probabilities are then calculated (lines 6-8), and the indices of the selected records from the secondary table, referred to as \textit{secondary records}, are output (line 9). The primary record and secondary records are subsequently passed into an attention-based soft alignment module (lines 10-11) for further training to derive the final label (line 12) and loss (line 13), as illustrated in Figure~\ref{fig:leal-framework}. The pretrained encoder, cluster centroids, feed-forward layers, and attention layers are optimized jointly (line 14). The detailed training process is presented in Algorithm \ref{alg:training}.

\begin{algorithm}[t]
\caption{Training Process of Leal}
\label{alg:training}

\begin{algorithmic}[1]
\REQUIRE Primary table \(\mathbf{X}^P\), secondary table \(\mathbf{X}^S\), labels \(\mathbf{y}\), batch size \(B\), number of clusters \(C\), number of candidate records \(K\), learning rate \(\eta\).
\OUTPUT Trained model parameters \(\theta\)
\STATE Initialize model parameters \(\theta\) randomly
\STATE Train encoder-decoder \((g^S, \phi^S)\) by minimizing \(\|\mathbf{X}^S - \phi^S(g^S(\mathbf{X}^S))\|_2^2\)
\STATE Initialize cluster centroids \(\mathbf{C}\) via K-means on \(g^S(\mathbf{X}^S)\)
\FOR{each epoch}
    \FOR{each batch \(\{x^P_b, y_b\}_{b=1}^B\) from \((\mathbf{X}^P, \mathbf{y})\)}
        \STATE Calculate cluster weights: \(\mathbf{w}_b \leftarrow \text{SoftMax}(\text{FeedForward}(x^P_b))\)
        \STATE Calculate in-cluster probabilities \(\mathbf{q}\) using Student's t-distribution
        \STATE Obtain sampling probabilities \(\mathbf{p}_b \leftarrow \mathrm{MLP}(\mathbf{q} \mathbf{w}_i^T)\)
        \STATE Sample \(K\) secondary records \(\mathbf{x}^S_b\) according to \(\mathbf{p}_b\)
        \STATE Encoding: \(z^P_b \leftarrow f^P(x^P_b)\), \(\mathbf{z}^S_b \leftarrow f^S(\mathbf{x}^S_b)\)
        \STATE Compute attention weights \(\boldsymbol{\lambda}_b\) and embedding \(\tilde{z}^P_b\)
        \STATE Predict labels: \(\hat{y}_b \leftarrow f(\theta; \tilde{z}^P_b)\)
        \STATE Calculate loss: \(\mathcal{L}_b \leftarrow \mathcal{L}(\hat{y}_b, y_b)\)
        \STATE Update parameters \(\theta \leftarrow \theta - \eta \nabla_{\theta} \mathcal{L}_b\)
    \ENDFOR
\ENDFOR
\end{algorithmic}
\end{algorithm}

\paragraph{Inference.} During inference, the pretrained encoder, cluster centroids, and feed-forward layer remain fixed. The procedure follows the forward pass of the supervised stage, with the only variation being in the cluster sampler. To ensure determinism, the top records with the highest sampling probabilities $\mathbf{p}_i$ are selected.

\paragraph{Scalability.} Suppose that training an individual table requires $M_0$ memory and $T_0$ time. Let $K$ represent the number of candidate secondary records. The current memory overhead of Leal is $K M_0$, as for each primary record, $K$ neighboring records are fed into the training process. The time complexity of Leal is $O(K^2 T_0)$, which arises from the quadratic complexity of the attention mechanism. However, with the cluster sampler, $K$ is typically less than 100, making the training time and memory overhead of Leal manageable.

\section{Experiments}\label{sec:experiments}
\subsection{Experimental Setup}

We conduct experiments on five real-world and five synthetic datasets, covering classification and regression tasks. Each real-world dataset comprises two tables from different sources, with shared features removed. The \texttt{house} dataset combines data from Lianjia \cite{qiu2017} and Airbnb \cite{airbnb2019} for house price prediction. The \texttt{bike} dataset uses Citibike \cite{citibike2016} and New York City Taxi routes \cite{nytlc2016} for travel time prediction. The \texttt{hdb} dataset integrates HDB resale prices \cite{hdb2018} and school rankings \cite{salary2020} in Singapore for resale price prediction. The \texttt{accidents} dataset, derived from Slovenian police traffic records \cite{slovenian_police_traffic_safety}, consists of individual and accident-level tables for accident type classification. The \texttt{hepatitis} dataset, sourced from the PKDD'02 Discovery Challenge database \cite{berka2002hepatitis}, consists of medical check tables for distinguishing Hepatitis B and C cases. Any shared features, if present, are removed from secondary tables. Dataset statistics are summarized in Table~\ref{tab:real_world_dataset_stats}, where ``cls'' means classification and ``reg'' means regression.

The synthetic datasets are generated from the UCI repository and include \texttt{breast} \cite{zwitter1988breast}, \texttt{covertype} \cite{blackard1998covertype}, \texttt{gisette} \cite{guyon2004gisette}, \texttt{letter} \cite{slate1991letter}, and \texttt{superconduct} \cite{hamidieh2018superconductivity}. Each synthetic dataset is randomly divided by features into two tables, serving as primary and secondary tables, without any overlapped feature. Detailed statistics of the synthetic datasets are provided in Table~\ref{tab:synthetic_dataset_stats}.

We apply one-hot encoding to categorical features, and normalize all features of real-world datasets to mitigate the impact of extreme values. Therefore, all features are treated as numerical in our experiments.

\begin{table}[ht]
  \centering
  \caption{Details of real-world datasets}
  \vskip 0.15in
  \small
  \scalebox{0.82}{
  \begin{tabular}{lcccccc}
  \toprule
  \multirow{2}{*}{\textbf{Dataset}} & \multicolumn{2}{c}{\textbf{Primary}} & \multicolumn{2}{c}{\textbf{Secondary}} & \multirow{2}{*}{\textbf{\#Class}} & \multirow{2}{*}{\textbf{Task}} \\
  \cmidrule(lr){2-3} \cmidrule(lr){4-5}
  & \textbf{\#Inst.} & \textbf{\#Feat.} & \textbf{\#Inst.} & \textbf{\#Feat.} & & \\
  \midrule
accidents & 237337 & 21 & 419818 & 16 & 6 & cls \\
bike & 200000 & 966 & 100000 & 4 & 1 & reg \\
hdb & 92065 & 70 & 165 & 9 & 1 & reg \\
hepatitis & 621 & 6 & 5691 & 10 & 2 & cls \\
house & 141049 & 54 & 27827 & 23 & 1 & reg \\
  \bottomrule
  \end{tabular}}
  \label{tab:real_world_dataset_stats}
\vspace{-5pt}
\end{table}

\begin{table}[ht]
  \centering
  \caption{Details of synthetic datasets}
  \vskip 0.15in
  \small
  \begin{tabular}{lcccc}
  \toprule
  \textbf{Dataset} & \textbf{\#Inst.} & \textbf{\#Feat.} & \textbf{Classes} & \textbf{Task} \\
  \midrule
breast & 286 & 43 & 2 & cls \\
covertype & 581012 & 54 & 7 & cls \\
gisette & 6000 & 5000 & 2 & cls \\
letter & 20000 & 15 & 26 & cls \\
superconduct & 21263 & 81 & 1 & reg \\
  \bottomrule
  \end{tabular}
  \label{tab:synthetic_dataset_stats}
\end{table}

\paragraph{Training.} All models are trained using the AdamW~\cite{loshchilov2017decoupled} optimizer with early stopping and a learning rate of 0.001. A batch size of 128 is employed, and training is conducted for up to 150 epochs, with early stopping determined by validation loss. The hyperparameters \(K\) and \(C\) are selected from the set \(\{1, 5, 10, 20, 100\}\), while the layer depth is chosen from \(\{1, 3, 6\}\). In the unsupervised stage, the depth of the encoder and decoder is set to 2, except for the \texttt{hdb} dataset, where a depth of 1 is used to prevent overfitting. All embedding dimensions are fixed at 100.

\paragraph{Evaluation.} For all datasets, the primary table is divided into training, validation, and test sets in a 7:1:2 ratio. The secondary table is utilized during both training and inference. Evaluation metrics include accuracy for classification tasks and root mean squared error (RMSE) for regression tasks. The mean and standard deviation of test scores are reported over five seeds.

\paragraph{Baselines.} We compare the proposed method with state-of-the-art deep tabular learning methods. "Solo" denotes training on the primary table only. The baselines are:
\begin{itemize}
    \item \textbf{Solo-MLP:} A three-layer MLP with hidden sizes of (800, 400, 400) and ReLU activation.
    \item \textbf{Solo-ResNet~\cite{gorishniy2021revisiting}:} An MLP model incorporating skip connection, ReLU activation and batch normalization.
    \item \textbf{Solo-FTTrans~\cite{gorishniy2021revisiting}:} A transformer-based model that embeds each feature as a token.
    \item \textbf{Solo-TabM~\cite{gorishniy2024tabm}:} A state-of-the-art multi-layer deep ensemble-based model.
\end{itemize}

\paragraph{Environement.} Each task is executed on a single NVIDIA V100 GPU with 32GB memory and an Intel(R) Xeon(R) Platinum 8168 CPU @ 2.70GHz. The system is equipped with 1.48TB of memory, which is not fully utilized. The code is implemented using PyTorch 2.5 and Python 3.10.

\subsection{Performance}

\begin{table*}[ht]
  \centering
  \caption{Performance of Leal and baselines on real-world datasets}\label{tab:real-perf}
  \vskip 0.15in
  \small
  \begin{tabular}{lcccccc}
  \toprule
  
  \multirow{2}{*}{\textbf{Methods}} & \multicolumn{5}{c}{\textbf{Real-world Dataset}} \\
  \cmidrule{2-6}
   & \textbf{accidents} (↑) & \textbf{bike} (↓) & \textbf{hdb} (↓) & \textbf{hepatitis} (↑) & \textbf{house} (↓)\\
  \midrule
  
Solo-MLP & 73.65\% \textpm 0.08\% & 541.4278 \textpm 1.5190 & 298.3065 \textpm 3.2910 & 65.74\% \textpm 11.74\% & 215.8363 \textpm 0.4579 \\
Solo-ResNet & \underline{73.89\%} \textpm 0.08\% & \textbf{242.7551} \textpm 0.9368 & 35.2518 \textpm 0.1684 & \underline{76.80\%} \textpm 4.02\% & 72.2510 \textpm 0.3008 \\
Solo-TabM & 73.69\% \textpm 0.05\% & 253.4289 \textpm 0.6411 & \underline{35.2030} \textpm 0.1276 & 74.40\% \textpm 3.39\% & 75.5523 \textpm 0.4723 \\
Solo-FTTrans & 73.47\% \textpm 0.10\% & 301.9098 \textpm 26.4545 & 35.4862 \textpm 0.2513 & 74.24\% \textpm 4.39\% & \underline{70.6538} \textpm 0.6250 \\
\midrule
Leal & \textbf{73.92\%} \textpm 0.16\% & \underline{248.0001} \textpm 1.5935 & \textbf{34.8088} \textpm 0.3468 & \textbf{79.03\%} \textpm 2.39\% & \textbf{51.7545} \textpm 0.2352 \\
  \bottomrule
\end{tabular}
\vspace{-5pt}
\end{table*}

\begin{table*}[ht]
  \centering
  \caption{Performance of Leal and baselines on synthetic datasets (OOM: Out of memory)}\label{tab:syn-perf}
  \vskip 0.15in
  \small
  \begin{tabular}{lcccccccccccc}
  \toprule
  \multirow{2}{*}{\textbf{Methods}} & \multicolumn{5}{c}{\textbf{Synthetic Dataset}} \\
  \cmidrule{2-6}
   & \textbf{breast} (↑) & \textbf{covertype} (↑) & \textbf{gisette} (↑) & \textbf{letter} (↑) & \textbf{superconduct} (↓)\\
  \midrule

Solo-MLP & 86.55\% \textpm 3.34\% & 58.60\% \textpm 13.30\% & 96.52\% \textpm 0.83\% & 44.88\% \textpm 5.01\% & 11.5641 \textpm 5.4981 \\
Solo-ResNet & \underline{89.65\%} \textpm 4.50\% & 73.75\% \textpm 0.12\% & \underline{97.30\%} \textpm 0.39\% & \textbf{89.55\%} \textpm 0.50\% & 0.1563 \textpm 0.1459 \\
Solo-TabM & 86.55\% \textpm 3.34\% & \underline{73.81\%} \textpm 0.19\% & 96.73\% \textpm 0.44\% & 83.97\% \textpm 0.49\% & 0.3011 \textpm 0.0739 \\
Solo-FTTrans & 87.59\% \textpm 3.68\% & 60.86\% \textpm 1.83\% & OOM & \underline{89.48\%} \textpm 0.46\% & \underline{0.1534} \textpm 0.1460 \\
\midrule
Leal & \textbf{93.11\%} \textpm 3.08\% & \textbf{76.46\%} \textpm 4.68\% & \textbf{97.57\%} \textpm 0.44\% & 87.57\% \textpm 1.28\% & \textbf{0.1468} \textpm 0.1505 \\
  \bottomrule
\end{tabular}%
\vspace{-5pt}
\end{table*}

The performance of Leal and the baseline methods across various datasets is summarized in Table~\ref{tab:real-perf} for real-world datasets and Table~\ref{tab:syn-perf} for synthetic datasets. Two key findings emerge from these results. First, Leal outperforms state-of-the-art approaches trained solely on the primary table for the majority of datasets. For instance, on the \texttt{house} dataset, Leal reduces RMSE by at least 26.8\% compared to all baselines. Second, Leal demonstrates significantly better performance on real-world datasets compared to synthetic datasets. This is attributed to the presence of fuzzy and many-to-many alignments in real-world datasets, which amplify the effectiveness of soft alignment by better capturing complex relationships and dependencies between data records, as also observed in \cite{wu2024federated}. In contrast, synthetic datasets typically involve only one-to-one alignments, limiting the benefits of soft alignment. Notably, FT-Transformer (FTTrans) fails to handle the high-dimensional \texttt{gisette} dataset due to its encoding approach, which treats each feature as a separate token. This results in an out-of-memory error on all tested GPUs.

\subsection{Efficiency}\label{subsec:efficiency}

Table~\ref{tab:train-time} presents the average training time per epoch for Leal and baseline methods, with $K=20$ and $C=20$, a typical hyperparameter setting for Leal in our experiments. Despite addressing the complex alignment challenges inherent in its design, Leal maintains competitive training efficiency compared to baselines trained on a single table. On average, Leal incurs a computational overhead ranging from 1.27\(\times\) to 22.82\(\times\), even with the quadratic growth in the number of record pairs introduced by incorporating the secondary table. This efficiency is largely attributed to Leal’s cluster sampler module, which significantly reduces the number of candidate pairs in the soft alignment process, thereby mitigating the computational burden. 

The overhead of Leal on large datasets is reasonable. For instance, on the \texttt{covertype} dataset with 581K instances in both tables, the alignment-training process for Leal requires approximately 692 minutes ($415 \times 100 / 60$ min). For comparison, adopting state-of-the-art approximate nearest neighbor search methods \textit{without learning}, such as FAISS-IVF \cite{douze2024faiss}, with a recall $\leq 10^{-2}$, would take around 312 minutes for a the same table scale according to a recent benchmark \cite{aumuller2020ann}. This highlights that while Leal introduces additional computational costs, its efficiency remains reasonable for large-scale datasets.

\begin{table*}[t!]
    \centering
    \caption{Average training time per epoch in seconds (OOM: out of memory)}\label{tab:train-time}
    \small
    \vskip 0.15in
    \begin{tabular}{lcccccccccc}
    \toprule
    \multirow{2}{*}{\textbf{Method}} & \multicolumn{10}{c}{\textbf{Dataset}} \\
    \cmidrule{2-11}
    & accidents & bike & breast & covertype & gisette & hdb & hepatitis & house & letter & superconduct\\
    \midrule
    Solo-MLP & 3.41 & 9.12 & 1.00 & 18.21 & 1.26 & 4.88 & 1.00 & 7.23 & 2.14 & 2.22 \\
    Solo-ResNet & 16.59 & 18.59 & 1.00 & 46.39 & 1.39 & 9.08 & 1.00 & 14.59 & 3.03 & 3.18 \\
    Solo-TabM & 13.09 & 14.76 & 1.00 & 29.06 & 1.22 & 7.83 & 1.00 & 10.78 & 2.51 & 2.62 \\
    Solo-FTT & 11.80 & 244.07 & 1.00 & 28.83 & OOM & 20.15 & 1.00 & 15.95 & 1.36 & 1.40 \\
    \midrule
    Leal & 69.64 & 50.04 & 1.03 & 415.43 & 1.85 & 20.17 & 1.27 & 34.85 & 8.52 & 3.62 \\
    Mean Overhead & 20.40$\times$ & 5.49$\times$ & 1.03$\times$ & 22.82$\times$ & 1.51$\times$ & 4.13$\times$ & 1.27$\times$ & 4.82$\times$ & 6.27$\times$ & 2.59$\times$ \\
    \bottomrule
    \end{tabular}
\end{table*}

\subsection{Ablation Studies}

\paragraph{Effect of Soft Alignment.}
We evaluate the effectiveness of soft alignment through a controlled experiment. To eliminate the influence of sampling techniques, we select $K$ candidates, including one ground-truth correctly aligned secondary record and $K-1$ randomly sampled secondary records. This setup allows us to assess the attention mechanism's ability to identify the correct secondary record. Each experiment is conducted under five seeds, and we report the mean and standard deviation of the accuracy in Figure~\ref{fig:lealtop}.

Two key findings emerge from Figure~\ref{fig:lealtop}. First, soft alignment effectively identifies the correct alignment at a modest $K$ value, despite the noise introduced by random sampling. For example, to outperform Solo-ResNet, the \texttt{covertype} dataset supports $K \leq 320$, while the \texttt{letter} dataset supports $K \leq 10$. Second, a comparison of Figure~\ref{fig:lealtop} and Table~\ref{tab:syn-perf} reveals that the performance of Leal in Figure~\ref{fig:lealtop} with small $K$ significantly exceeds its performance in Table~\ref{tab:syn-perf}. This indicates that the primary bottleneck for Leal lies in the clustering process. Enhancing the clustering process to improve the probability of including the correct pair in the candidate set would further boost Leal’s performance.

\begin{figure}[ht]
    \centering
    \includegraphics[width=0.49\linewidth]{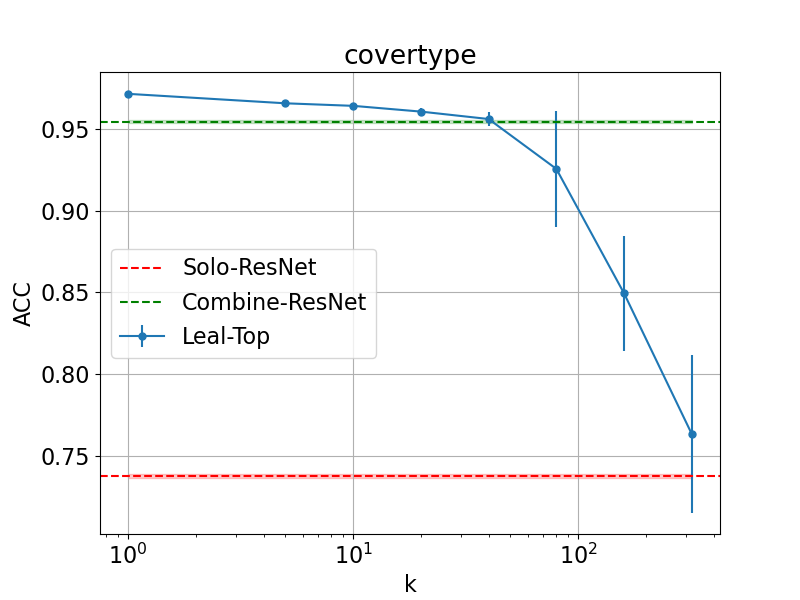}
    \includegraphics[width=0.49\linewidth]{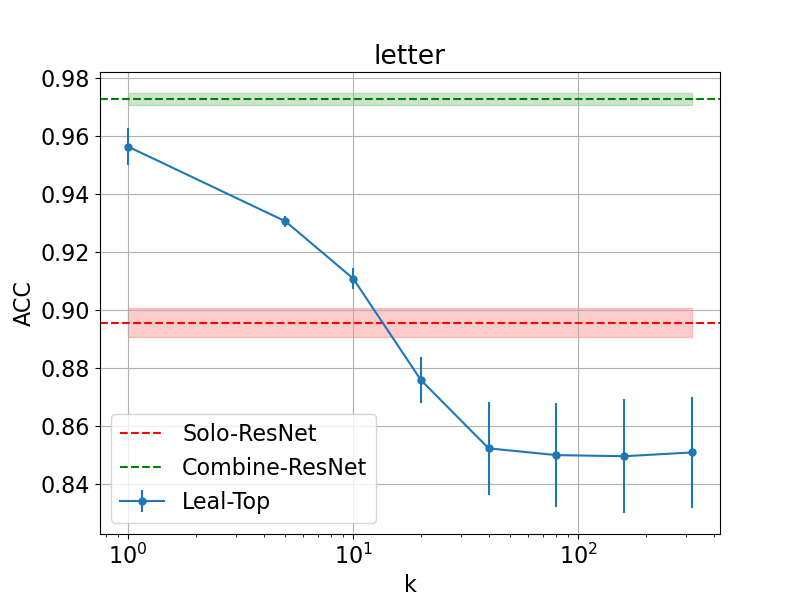}
    \caption{Capacity of attention mechanism in soft alignment}
    \label{fig:lealtop}
\end{figure}

\paragraph{Effect of Hyperparameters.}
We conduct ablation studies to examine the impact of the number of clusters ($C$) and the number of neighbors ($K$) on Leal's performance. The results are illustrated in Figure~\ref{fig:ablation-C} and Figure~\ref{fig:ablation-K}, respectively. From these results, we derive three key findings: First, setting $K$ too large typically has a negative effect, consistent with the observations in Figure~\ref{fig:lealtop}. Second, the effect of $C$ is dataset-size-dependent. For instance, the large \texttt{bike} dataset benefits from a larger $C$ value, whereas for the small \texttt{hepatitis} dataset, a small $C$ value is sufficient. Third, the effect of $C$ also interacts with the number of neighbors ($K$). For example, in the \texttt{hepatitis} dataset (Figure~\ref{fig:ablation-C}), when $K$ is small, a larger $C$ value is advantageous; however, when $K$ is large, a smaller $C$ value proves to be more effective. These findings suggest that soft alignment and cluster sampling can complement each other during the training.

\begin{figure}
  \centering
  \includegraphics[width=0.48\linewidth]{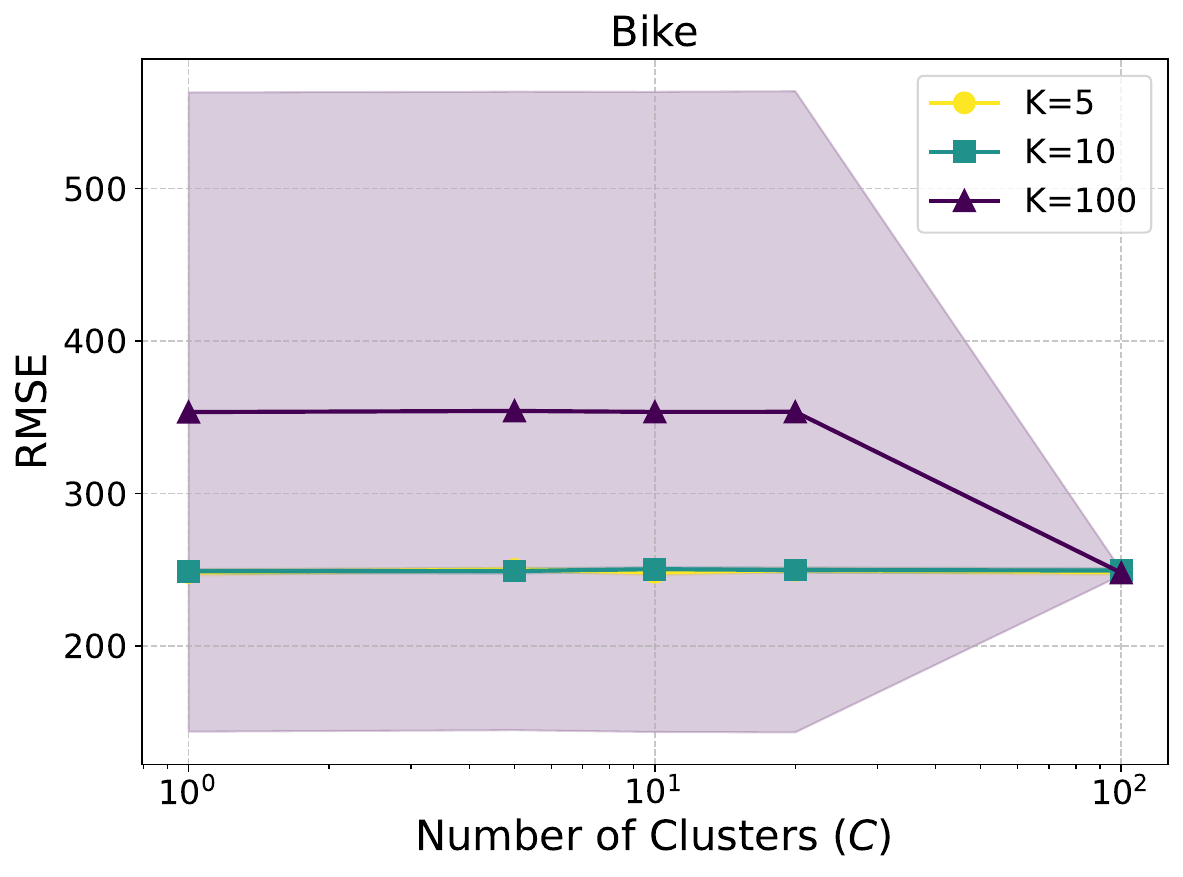}
  \includegraphics[width=0.48\linewidth]{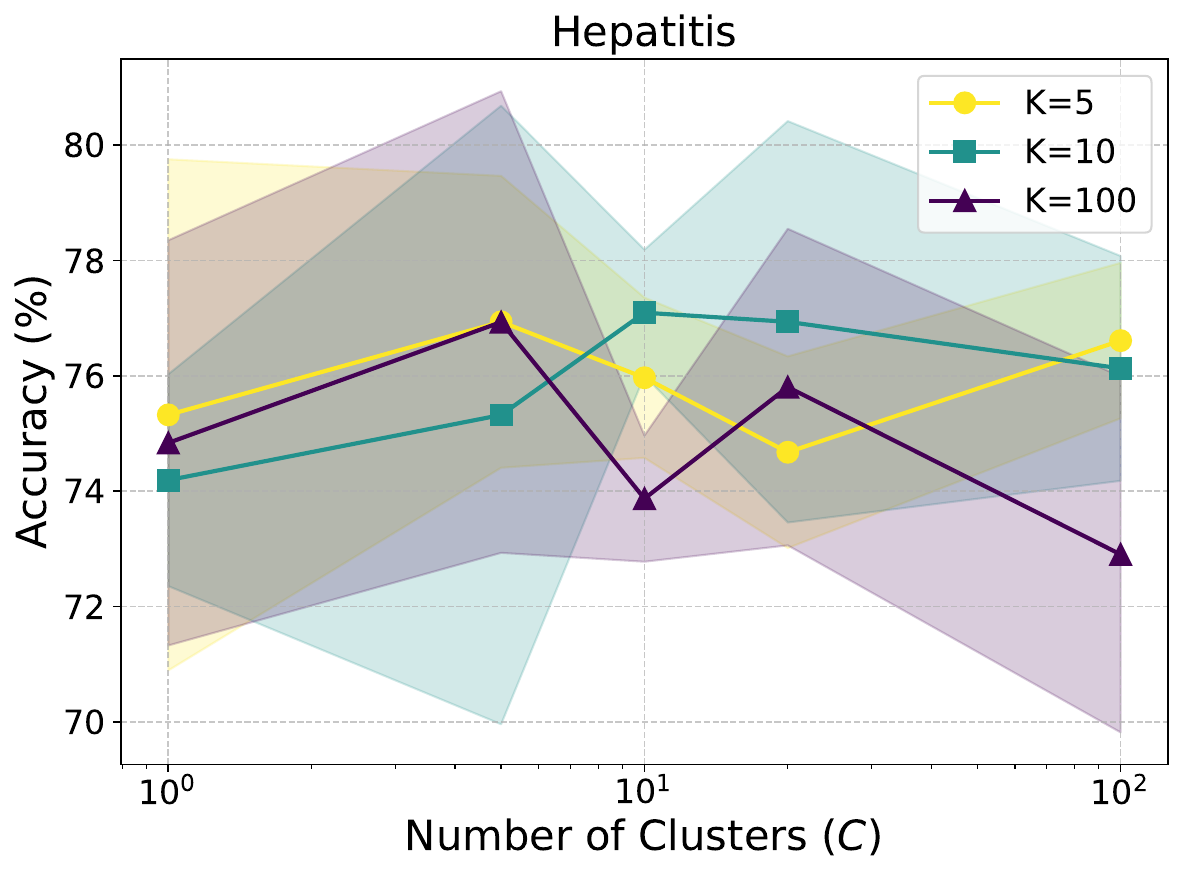}
  \caption{Effect of cluster size $C$ on performance}
  \label{fig:ablation-C}
\end{figure}

\begin{figure}
  \centering
  \includegraphics[width=0.48\linewidth]{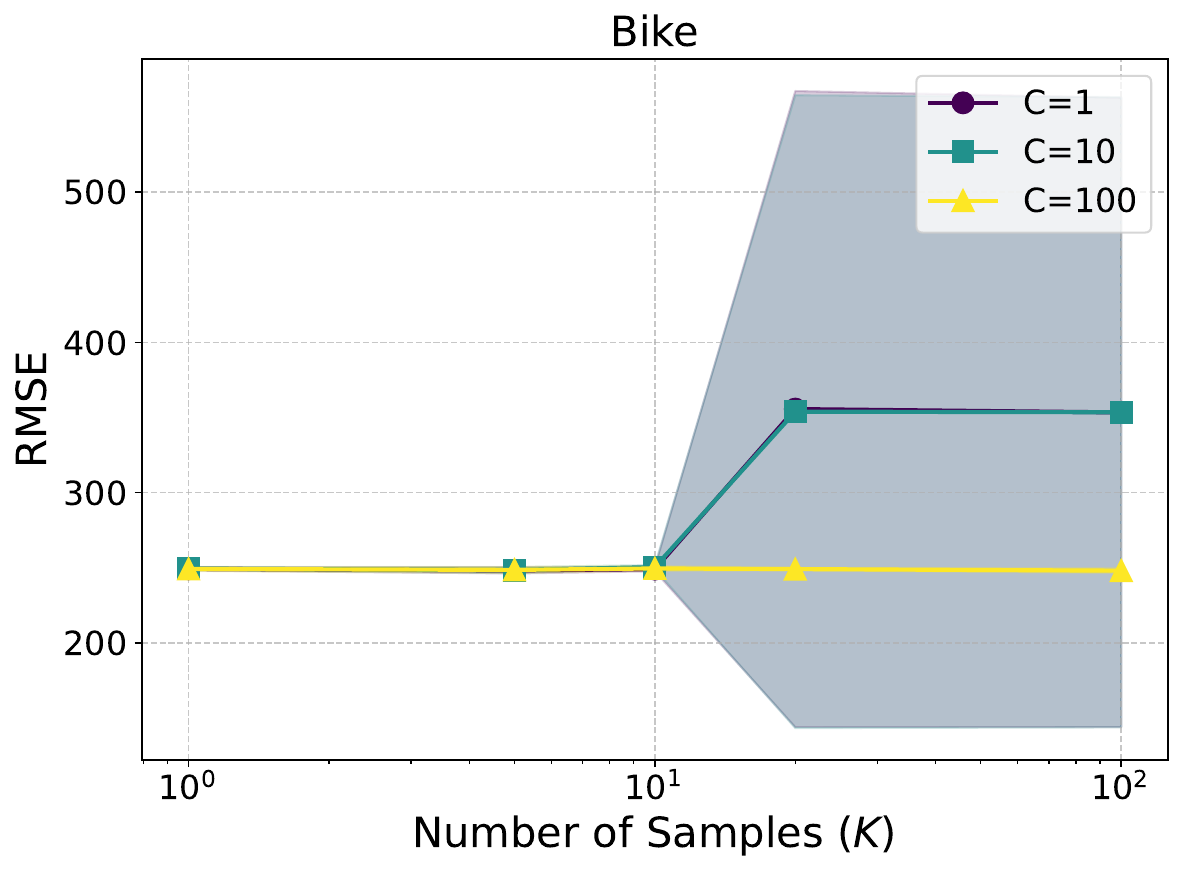}
  \includegraphics[width=0.48\linewidth]{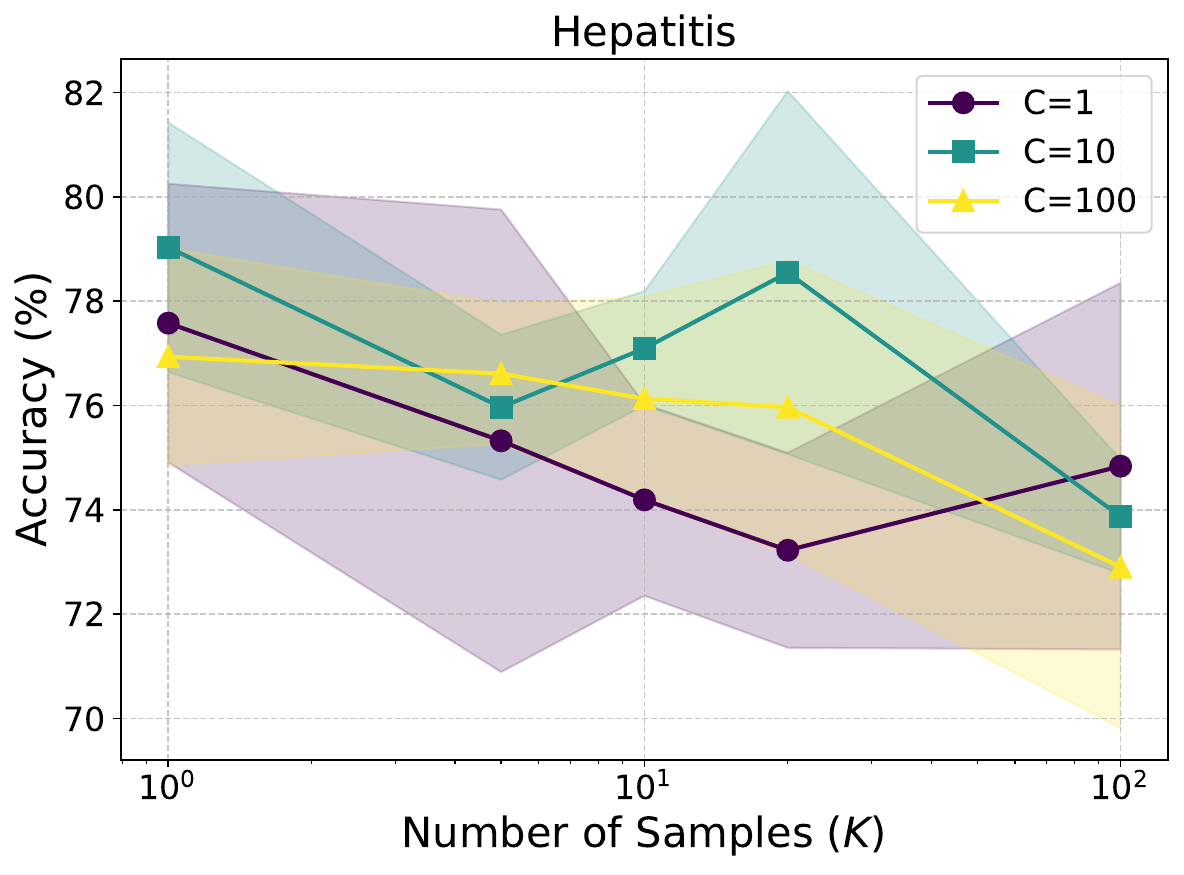}
  \caption{Effect of number of neighbors $K$ on performance}
  \label{fig:ablation-K}
\end{figure}

\section{Future Work}\label{sec:future-work}
In the future, we aim to extend the model to handle multiple relational tables. The primary challenges are overfitting and efficiency due to the involvement of numerous tables. Successfully addressing these challenges would enable knowledge fusion across real-world tabular data without shared features, supporting more applications in unsupervised multi-modal learning and vertical federated learning.

Additionally, we aim to develop efficient methods for identifying relationships between tables. While training a Leal model to evaluate table relationships by performance is feasible, it is computationally expensive and lacks scalability. Exploring efficient approaches could significantly enhance knowledge integration across tables.

\section{Conclusion}\label{sec:conclusion}
In this paper, we theoretically and empirically demonstrate that properly aligned data are more effectively learned by machine learning models compared to misaligned data. Building on this observation, we propose an integrated model, \emph{Leal}, which seamlessly combines soft alignment and learning for tabular data. We conduct extensive experiments on both synthetic and real-world datasets to validate the effectiveness of Leal. Looking ahead, we believe that latent alignment learning represents a promising direction for advancing machine learning to leverage the widely available heterogeneous tabular data without shared features.

\section*{Impact Statement}

This paper presents work whose goal is to advance the field of 
Machine Learning. There are many potential societal consequences 
of our work, none which we feel must be specifically highlighted here.

\bibliography{references}
\bibliographystyle{icml2025}

\newpage
\appendix
\onecolumn
\section{Proof}\label{sec:proof}
\subsection{Error Gap Between Aligned and Misaligned Data}\label{subsec:proof-align-misalign}

\thmalignment*

\begin{proof}

For the aligned case, we can derive the mean squared error (MSE) as follows:
\begin{equation}\label{eq:mse_aligned}
    \mathrm{MSE}_\mathrm{aligned} = \inf_{\boldsymbol{\alpha} \in R^{m^P}, \boldsymbol{\beta} \in R^{m^S}} \|\mathbf{y} - \mathbf{X}^P \boldsymbol{\alpha} - \mathbf{X}^S \boldsymbol{\beta}\|
\end{equation}
The ordinary least squares (OLS) estimator of $\boldsymbol{\alpha}$ is given by:
\begin{equation}
    \hat{\boldsymbol{\alpha}} := (\mathbf{X}^{P \top} \mathbf{X}^P)^{-1} \mathbf{X}^P (\mathbf{y} - \mathbb{E}[\mathbf{R}] \mathbf{X}^S \boldsymbol{\beta}) 
\end{equation}
For a permutation matrix $\mathbf{R}$ under uniform distribution, we have $\mathbb{E}[\mathbf{R}] = \frac{1}{n}\mathds{1}^\top \mathds{1}$. Therefore:
\begin{equation}\label{eq:alpha_hat}
    \hat{\boldsymbol{\alpha}} = (\mathbf{X}^{P \top} \mathbf{X}^P)^{-1} \mathbf{X}^P (\mathbf{y} - \frac{1}{n} \mathds{1}^\top \mathds{1} \mathbf{X}^S \boldsymbol{\beta}) 
\end{equation}
The MSE for the misaligned case can be expressed as:
\begin{align}
    \mathrm{MSE}_{\mathrm{misaligned}} 
    & = \inf_{\boldsymbol{\beta}} \inf_{\boldsymbol{\alpha}} \mathbb{E}_\mathbf{R} \|\mathbf{y} - \mathbf{X}^P \boldsymbol{\alpha} - \mathbf{R} \mathbf{X}^S \boldsymbol{\beta}\|_2^2 \\
    & = \inf_{\boldsymbol{\beta}} \mathbb{E}_\mathbf{R} \|\mathbf{y} - \mathbf{X}^P \hat{\boldsymbol{\alpha}} - \mathbf{R} \mathbf{X}^S \boldsymbol{\beta}\|_2^2 \\
\end{align}
Substituting $\hat{\boldsymbol{\alpha}}$ from equation~\ref{eq:alpha_hat}, we obtain:
\begin{align}
    \mathrm{MSE}_{\mathrm{misaligned}} 
    & = \inf_{\boldsymbol{\beta}} \mathbb{E}_\mathbf{R} \left\|\mathbf{y} - \mathbf{X}^P (\mathbf{X}^{P \top} \mathbf{X}^P)^{-1} (\mathbf{X}^P \mathbf{y} - \mathbf{X}^P \frac{1}{n} 1^\top 1 \mathbf{X}^S \boldsymbol{\beta}) - \mathbf{R} \mathbf{X}^S \boldsymbol{\beta}\right\|_2^2 \\
    & = \inf_{\boldsymbol{\beta}} \mathbb{E}_\mathbf{R} \left\| (\mathbf{I} - \mathbf{X}^P (\mathbf{X}^{P \top} \mathbf{X}^P)^{-1} \mathbf{X}^P)\mathbf{y} + (\mathbf{X}^P (\mathbf{X}^{P \top} \mathbf{X}^P)^{-1} \mathbf{X}^P \frac{1}{n} \mathds{1}^\top \mathds{1} \mathbf{X}^S \boldsymbol{\beta}) - \mathbf{R} \mathbf{X}^S \boldsymbol{\beta}\right\|_2^2 
\end{align}
Since $\mathbf{X}^P (\mathbf{X}^{P \top} \mathbf{X}^P)^{-1} \mathbf{X}^P$ is a projection matrix that projects any vector onto the column space of $\mathbf{X}^P$, and $\mathbf{X}^S \boldsymbol{\beta}$ is orthogonal to the column space of $\mathbf{X}^P$, the term $\mathbf{X}^P (\mathbf{X}^{P \top} \mathbf{X}^P)^{-1} \mathbf{X}^P \frac{1}{n} \mathds{1}^\top \mathds{1} \mathbf{X}^S \boldsymbol{\beta} = 0$. Thus:
\begin{align}
    \mathrm{MSE}_{\mathrm{misaligned}}
    & = \inf_{\boldsymbol{\beta}} \mathbb{E}_\mathbf{R} \left\| (\mathbf{I} - \mathbf{X}^P (\mathbf{X}^{P \top} \mathbf{X}^P)^{-1} \mathbf{X}^P)\mathbf{y} - \mathbf{R} \mathbf{X}^S \boldsymbol{\beta}\right\|_2^2 \\
    & = \inf_{\boldsymbol{\beta}} \mathbb{E}_\mathbf{R} \left[\left\|\mathbf{R} \mathbf{X}^S \boldsymbol{\beta}\right\|_2^2 - 2\left[(\mathbf{I} - \mathbf{X}^P (\mathbf{X}^{P \top} \mathbf{X}^P)^{-1} \mathbf{X}^P)\mathbf{y}\right]^\top \mathbf{R} \mathbf{X}^S \boldsymbol{\beta} + \left\|(\mathbf{I} - \mathbf{X}^P (\mathbf{X}^{P \top} \mathbf{X}^P)^{-1} \mathbf{X}^P)\mathbf{y}\right\|_2^2\right]
\end{align}
By properties of permutation matrices:
\begin{equation}
    \mathbb{E}_\mathbf{R}\| \mathbf{R} \mathbf{X}^S \boldsymbol{\beta}\|_2^2 = \|\mathbf{X}^S \boldsymbol{\beta}\|_2^2; \; \mathbb{E}_\mathbf{R} [\mathbf{R}]= \frac{1}{n}\mathds{1}^\top \mathds{1}
\end{equation}
Therefore:
\begin{align}
    \mathrm{MSE}_{\mathrm{misaligned}}
    & = \inf_{\boldsymbol{\beta}} \left[\left\|\mathbf{X}^S \boldsymbol{\beta}\right\|_2^2 - 2\left[(\mathbf{I} - \mathbf{X}^P (\mathbf{X}^{P \top} \mathbf{X}^P)^{-1} \mathbf{X}^P)\mathbf{y}\right]^\top \frac{1}{n}\mathds{1}^\top \mathds{1} \mathbf{X}^S \boldsymbol{\beta} + \left\|(\mathbf{I} - \mathbf{X}^P (\mathbf{X}^{P \top} \mathbf{X}^P)^{-1} \mathbf{X}^P)\mathbf{y}\right\|_2^2\right]
\end{align}
Since $\mathbf{I} - \mathbf{X}^P (\mathbf{X}^{P \top} \mathbf{X}^P)^{-1} \mathbf{X}^P$ projects any vector onto the orthogonal complement of the column space of $\mathbf{X}^P$, the term $\left[(\mathbf{I} - \mathbf{X}^P (\mathbf{X}^{P \top} \mathbf{X}^P)^{-1} \mathbf{X}^P)\mathbf{y}\right]^\top \frac{1}{n}\mathds{1}^\top \mathds{1} \mathbf{X}^S \boldsymbol{\beta} = 0$. Hence:
\begin{align}
    \mathrm{MSE}_{\mathrm{misaligned}}
    & = \inf_{\boldsymbol{\beta}} \left[\left\|\mathbf{X}^S \boldsymbol{\beta}\right\|_2^2 + \left\|(\mathbf{I} - \mathbf{X}^P (\mathbf{X}^{P \top} \mathbf{X}^P)^{-1} \mathbf{X}^P)\mathbf{y}\right\|_2^2\right] \\
    & = \inf_{\boldsymbol{\beta}} \left\|\mathbf{X}^S \boldsymbol{\beta}\right\|_2^2 + \left\|(\mathbf{I} - \mathbf{X}^P (\mathbf{X}^{P \top} \mathbf{X}^P)^{-1} \mathbf{X}^P)\mathbf{y}\right\|_2^2 \\
\end{align}
The minimum is attained at $\boldsymbol{\beta} = \mathbf{0}$, yielding:
\begin{align}
    \mathrm{MSE}_{\mathrm{misaligned}}
    & = \left\|(\mathbf{I} - \mathbf{X}^P (\mathbf{X}^{P \top} \mathbf{X}^P)^{-1} \mathbf{X}^P)\mathbf{y}\right\|_2^2 \\
    & = \inf_{\boldsymbol{\alpha} \in \mathbb{R}^{m^P}, \boldsymbol{\beta} = \mathbf{0}} \left\|\mathbf{y} - \mathbf{X}^P \boldsymbol{\alpha} - \mathbf{X}^S \boldsymbol{\beta}\right\|_2^2 \\
\end{align}
Comparing with Equation~\ref{eq:mse_aligned}, we conclude:
\begin{equation}
    \mathrm{MSE}_{\mathrm{misaligned}} \geq \inf_{\boldsymbol{\alpha} \in \mathbb{R}^{m^P}, \boldsymbol{\beta} \in \mathbb{R}^{m^S}} \left\|\mathbf{y} - \mathbf{X}^P \boldsymbol{\alpha} - \mathbf{X}^S \boldsymbol{\beta}\right\|_2^2 = \mathrm{MSE}_{\mathrm{aligned}}
\end{equation}
\end{proof}

\subsection{Approximation Capacity of Cluster Sampler}\label{subsec:proof-cluster-sampler}

\begin{definition}[Definition of optimal cluster sampler]
    Assume the inputs are uniformly bounded by some constant $B$. 
    The optimal cluster sampler is defined by the uniform equi-continuous cluster sampler function which achieves the minimal optimization loss for the prediction task in \cref{fig:leal-framework}.
    \begin{equation}
        \textrm{Optimal cluster sampler} := \arginf_{\textrm{Uniform equi-continuous cluster sampler}} \textrm{Loss}(\textrm{cluster sampler})
    \end{equation}
    The cluster sampler is defined over bounded inputs ($|X^P|_{\infty} \leq B, |X^S|_{\infty} \leq B$) from $\mathbb{R}^{m^P} \times \mathbb{R}^{n^S \times m^S}$ and output in $\mathbb{R}^{n^S}$.
\end{definition}

\begin{remark}
    The existence of such optimal cluster sampler is guaranteed by the boundedness and uniform equi-continuity of the set of cluster sampler functions. 
\end{remark}

\thmclustersampler*

\begin{proof}
    We just need to prove the statement for small $\epsilon \leq 6$.

    The input of cluster sampler is $1 \times m^P$ and output is $n^S \times m^S$, the final prediction is to generate a sample probabilities:
    \begin{equation}
        (n^S * m^S, 1 * m^P) \to (n^S * d, 1 * C) \to (n^S * C, 1 * C) \to n^S * 1. 
    \end{equation}

    Also, since there is no weight depends on dimension $n_2$, we can reduce the approximation statement to that there exists trainable weight such that the continuous function $h$ can be approximated:
    \begin{equation}
        (1 * m^S, 1 * m^P) \to (n^S * d, 1 * C) \to (n^S * C, 1 * C) \to 1 * 1. 
    \end{equation}

    Notice that the layer operation of secondary embedding and trainable centroids weights $(C \times d)$ is continuous and the pretrained encoder as a neural network (which is a universal approximator) can approximates any continuous function $f$ composited with inverse embedding. 
    For simplicity, we will consider $m^P = m^S = 1$. 
    For any continuous function $h(p, s) \in [0, 1]$,
    we just need to show there exists trainable weight $\theta_1$, $\theta_2$ such that 
    \begin{equation}
        f(p; \theta_1) \odot g(s; \theta_2) = \sum_{i=1}^C f_i(p; \theta_1) \odot g_i(s; \theta_2). 
    \end{equation}
    Here $f(p; \theta_1) \in \mathbb{R}^C$ is a function of $p$ parameterized by $\theta_1$ and $g(s; \theta_1) \in \mathbb{R}^C$ is a function of $s$ parameterized by $\theta_2$.  
    As any continuous function $f(p, s)$ has a corresponding Taylor series expansion, it means for any $\epsilon > 0$, there exists $C$ which depends on error $\epsilon$ such that
    \begin{equation}
        \sup_p \sup_s |h(p, s) -\sum_{i=1}^C pol_{1,i}(p) pol_{2,i}(s)| \leq \frac{\epsilon}{2}. 
    \end{equation}
    Furthermore, as polynomial functions are continuous function, therefore $f_i$ can be used to approximate the polynomial function $pol_{1, i}$ and $g$ can be used to approximate the polynomial function $pol_{2, i}$.
    \begin{align}
        \sup_p |pol_{1,i}(p) - f_i(p; \theta_1)| & \leq \frac{\epsilon}{6B} \\ 
        \sup_s |pol_{2,i}(s) - g_i(s; \theta_2)| & \leq \frac{\epsilon}{6B}. 
    \end{align}
    Here $B := \max(1, \sup_p \max_{i} |pol_{1, i}(p)|, \sup_s \max_{i} |pol_{2, i}(s)|).$ 
    We show that the cluster sampler is capable to approximate any desirable continuous cluster sampler. 
    \begin{equation}
        \sup_p \sup_s |h(p, s) -\sum_{i=1}^C f_i(p; \theta_1) g_i(s; \theta_2)| \leq \frac{\epsilon}{2} + \frac{\epsilon}{6B} * B + \frac{\epsilon}{6B} (B + \frac{\epsilon}{6B}) = \frac{5}{6} \epsilon + \frac{\epsilon^2}{36B^2} < \epsilon. 
    \end{equation}
    The last inequality comes from $\epsilon < 6$. 
    The universal approximation capacity of the cluster sampler is proved. 
\end{proof}

\begin{remark}
    Since we are working with a cluster sampler with specific manually designed structure, it mainly comes from the fact the student's t-kernel introduce a suitable implicit bias to more efficiently learn the cluster sample probability $(n_2 \times 1)$. 
\end{remark}

\end{document}